\documentclass[letterpaper]{IEEEtran}
\usepackage{fullpage,graphicx,url,setspace}

\usepackage{xparse}
\usepackage{cite}
\usepackage{bm}
\usepackage{color}
\usepackage[small,bf]{caption}
\usepackage{amsmath,verbatim,amssymb,amsfonts,amscd, graphicx, algorithm, algorithmic}
\usepackage{amsmath,amsthm}
\usepackage{mathtools}
\mathtoolsset{mathic}
\usepackage{microtype}
\usepackage{wrapfig}

\renewcommand{\baselinestretch}{1}

\usepackage{ifxetex,ifluatex}
\ifxetex
\usepackage{xltxtra} 
\else
\ifluatex
\usepackage{luatextra}
\else
\usepackage{amssymb}
\usepackage{algorithmic, algorithm}
\usepackage[utf8]{inputenc}

\csname else\expandafter\endcsname
\romannumeral -`0
\fi
\fi
\iftrue\relax%
\defaultfontfeatures{Ligatures=TeX}
\usepackage[math-style=TeX, vargreek-shape=TeX]{unicode-math}
\fi

\NewDocumentCommand{\entropy}{om}{\mathbb{H}\left[#2
    \IfValueT{#1}{\,\middle|\,#1}\right]}
\NewDocumentCommand{\bentropy}{lm}
  {\widetilde{\mathbb{H}}#1\left[#2\right]}

\NewDocumentCommand{\mutualInfo}{omm}{\mathbb{I}\left[#2;#3
    \IfValueT{#1}{\,\middle|\,#1}\right]}

\usepackage{tikz}
\usetikzlibrary{shapes.geometric, positioning, shadows}
\usepackage{graphicx,color}

\newtheorem{theorem}{Theorem}

\newtheorem{lemma}{Lemma}

\newtheorem{remark}{Remark}

\newtheorem{corollary}{Corollary}

\DeclareMathOperator*{\argmin}{arg min}
\DeclareMathOperator*{\argmax}{arg max}

 \newcommand*{\transpose}[1]{#1\sp{\intercal}}

\begin{document}
%
\title{Sequential Sensing with Model Mismatch}


\author{Ruiyang~Song,~
       Yao~Xie,~
        and~Sebastian~Pokutta~
        \thanks{Ruiyang~Song (songry12@mails.tsinghua.edu.cn) is with the Dept. of Electronic Engineering, Tsinghua University, Beijing, China. }
\thanks{Yao~Xie (yao.xie@isye.gatech.edu) and Sebastian~Pokutta  (sebastian.pokutta@isye.gatech.edu) are H. Milton Stewart School of Industrial and Systems Engineering, 
Georgia Institute of Technology, Atlanta, GA.}
\thanks{This work is partially supported by NSF grant CMMI-1300144
and CCF-1442635. }
}

\maketitle

\begin{abstract}
We characterize the performance of sequential information guided sensing, Info-Greedy Sensing \cite{InfoGreedy2014}, when there is a mismatch between the true signal model and the assumed  model, which may be a sample estimate. In particular, we consider a setup where the signal is low-rank Gaussian and the measurements are taken in the directions of eigenvectors of the covariance matrix $\Sigma$ in a decreasing order of eigenvalues. 
We establish a set of performance bounds when a mismatched covariance matrix $\widehat{\Sigma}$ is used, in terms of the gap of signal posterior entropy, as well as the additional amount of power required to achieve the same signal recovery precision. Based on this, we further study how to choose an initialization for Info-Greedy Sensing using the sample covariance matrix, or using an efficient covariance sketching scheme. 

\end{abstract}


%
\IEEEpeerreviewmaketitle

\IEEEPARstart{}{} 

\section{Introduction}

Sequential compressed sensing is a promising new information acquisition and recovery technique to process big data that arise in various applications such as compressive imaging \cite{NeifeldTaskSpecific2008, KeAshok2010, NeifeldCSImaging2011}, 
 power network monitoring
  \cite{WirelessHouseElectricity2014}, and large scale sensor networks
  \cite{sparseSensorLocal2011}. The sequential nature of the problems arises either because the measurements are taken one after another, or due to the fact that the data is obtained in a streaming fashion so that it has to be processed in one pass. 
  
To harvest the benefits of adaptivity in sequential compressed sensing,  various
algorithms have been developed (see \cite{InfoGreedy2014} for a  review.) We may classify these algorithms as 
(1)
being agnostic about the signal distribution and, hence, using random measurements 
\cite{HauptAdaptiveCS2009, DavenportArias-Castro2012,TajerPoor2012,MalioutovSanghaviWillsky2010,HauptSeqCS2012,JainSoniHaupt2013,MalloyNowak2013}; 
(2) exploiting additional structure of the signal (such as graphical structure
\cite{KrishnamurthySingh13} and tree-sparse structure \cite{ErvinCastro2013, AkshayHaupt2014}) to design measurements; 
(3) exploiting the distributional information of the signal in choosing the measurements possibly through maximizing mutual information:  the seminal Bayesian compressive
  sensing work \cite{BayesianCS2008}, Gaussian mixture models (GMM) \cite{TaskDrivenDuarte2013, CarsonChenRodrigues2012} and our earlier work \cite{InfoGreedy2014} which presents a general framework for information guided sensing referred to as Info-Greedy Sensing.

In this paper we consider the setup of Info-Greedy Sensing\cite{InfoGreedy2014}, as it provides certain optimality guarantees. Info-Greedy Sensing aims at designing subsequent measurements to maximize the mutual information conditioned on previous measurements. 
Conditional mutual information is a natural metric here, as it captures exclusively useful new information between the signal and the result of the measurement disregarding noise and what has already been learned from previous measurements. 
It was shown in  \cite{InfoGreedy2014} that Info-Greedy Sensing for a Gaussian signal is equivalent to choosing the sequential measurement vectors $a_1, a_2, \ldots$ as the orthonormal eigenvectors of $\Sigma$ in a decreasing order of eigenvalues.

In practice, we do not know the signal covariance matrix $\Sigma$ and have to use a sample covariance matrix $\widehat{\Sigma}$ as an estimate. As a consequence, the measurement vectors are calculated from $\widehat{\Sigma}$, which deviate from the optimal directions. Since we almost always have to use some estimate for the signal covariance, it is important to quantify the performance of sensing algorithms with model mismatch. 

In this paper, we characterize the performance of Info-Greedy Sensing for Gaussian signals \cite{InfoGreedy2014} when the true signal covariance matrix is replaced with a proxy, which may be an estimate from direct samples or using a covariance sketching scheme. We establish a set of theoretical results including (1) relating the error in the covariance matrix $\|\Sigma -\widehat{\Sigma}\|$ to the entropy of the signal posterior distribution after each sequential measurement, and thus characterizing the gap between this entropy and the entropy when the correct covariance matrix is used; (2) establishing an upper bound on the amount of additional power required to achieve the same precision of the recovered signal if using an estimated covariance matrix; (3) if initializing Info-Greedy Sensing via a sample covariance matrix, finding the minimum number of samples required so that using such an initialization can achieve good performance; (4) presenting a covariance sketching scheme to initialize Info-Greedy Sensing and find the conditions so that using such an initialization is sufficient.  We also present a numerical example to demonstrate the good performance of Info-Greedy Sensing compared to a batch method (where measurements are not adaptive) when there is mismatch.

Our notations are standard. 
Denote $[n] \triangleq \{1,2,\ldots,n\}$; $\|X\|$ is the spectral norm of a matrix $X$, $\|X\|_F$ denotes the Frobenius norm of a matrix $X$, and $\|X\|_{*}$ represents the nuclear norm of a matrix $X$; $\|x\|$ is the $\ell_2$ norm of a vector $x$, and $\|x\|_1$ is the $\ell_1$ norm of a vector $x$; let $\chi_n^2$ be the quantile function of the chi-squared distribution with $n$ degrees of freedom; let $\mathbb{E}[x]$ and $\mbox{Var}[x]$ denote the mean and the variance of a random variable $x$; $X \succeq 0$ means that the matrix $X$ is positive semi-definite.

\section{Problem setup}

A typical sequential compressed sensing setup is as follows. Let $x \in \mathbb{R}^n$ be an unknown $n$-dimensional signal. We make $K$ measurements of $x$ sequentially
\[
y_k = \transpose a_k x + w_k, \quad k = 1, \ldots, K,
\]
and the power of the measurement is $\|a_k\|^2 = \beta_k$. The goal is to recover $x$ using measurements $\{y_k\}_{k=1}^K$. 
Consider a Gaussian signal $x \sim \mathcal{N}(0, \Sigma)$ with known zero mean and covariance matrix $\Sigma$ (here without loss of generality we have assumed the signal has zero mean). Assume the rank of $\Sigma$ is $s$ and the signal can be low-rank $s\ll n$. 
Info-Greedy Sensing \cite{InfoGreedy2014} chooses each measurement to maximizes the conditional mutual information 
\[
a_{k} \leftarrow \argmax_{a}
  \mutualInfo[y_{j}, a_{j}, j < k]{x}
  {\transpose{a} x + w}/\beta_k.
\]
The goal is to use minimum number of measurements (or total power) so that the estimated signal is recovered with precision $\varepsilon$: $\|\hat{x} - x\| < \varepsilon$ with high probabilities.  

In \cite{InfoGreedy2014}, we have devised a solution to the above problem, and established that Info-Greedy Sensing for low-rank Gaussian signal is to measure in the directions of the eigenvectors of $\Sigma$ in a decreasing order of eigenvalues with power allocation depending on the noise variance, signal recovery precision $\varepsilon$ and confidence level $p$, as given  in Algorithm \ref{alg:Gaussian-all}.  

Ideally, if we know the true signal covariance we will use the corresponding eigenvector to form measurements. However, in practice, we have to use an estimate of the covariance matrix which usually has errors. To establish performance bound when there is a mismatch between the assumed and the true covariance matrix, we adopt a metric which is the posterior entropy of the signal conditioned on previous measurement outcomes.  The entropy of a Gaussian signal $x \sim\mathcal{N}(\mu, \Sigma)$ is given by
  \[
  \entropy{x} = \frac{1}{2}\ln \left((2\pi e)^n \det(\Sigma)\right).
  \]
Hence, the conditional mutual information is essentially the log of the determinant of the conditional covariance matrix, or equivalently the log of the volume of the ellipsoid defined by the covariance matrix. Here, to accommodate the scenario where the covariance matrix is low-rank, we consider a modified definition for conditional entropy, which is the log of the volume of the ellipsoid on the low-dimensional space. Let $\Sigma_k$ be the underlying true signal covariance conditioned on the previous $k$ measurements; denote by $\widehat{\Sigma}_k$ the observed covariance matrix, which is also the output of the sequential algorithm. Assume the rank of $\Sigma$ is $s$. Then the metric we use to track the progress of our algorithm is
\[
\entropy[y_{j}, a_{j}, j \leq k]{x}
= \ln ((2\pi e)^{s/2} {\textsf{Vol}}(\Sigma_k)),
\] 
where ${\textsf{Vol}}(\Sigma_k)$ is the volume of the ellipse defined by the covariance matrix $\Sigma_k$, which is equal to the product of its non-zero eigenvalues.

\begin{figure}
\begin{center}
\includegraphics[width = 0.7\linewidth]{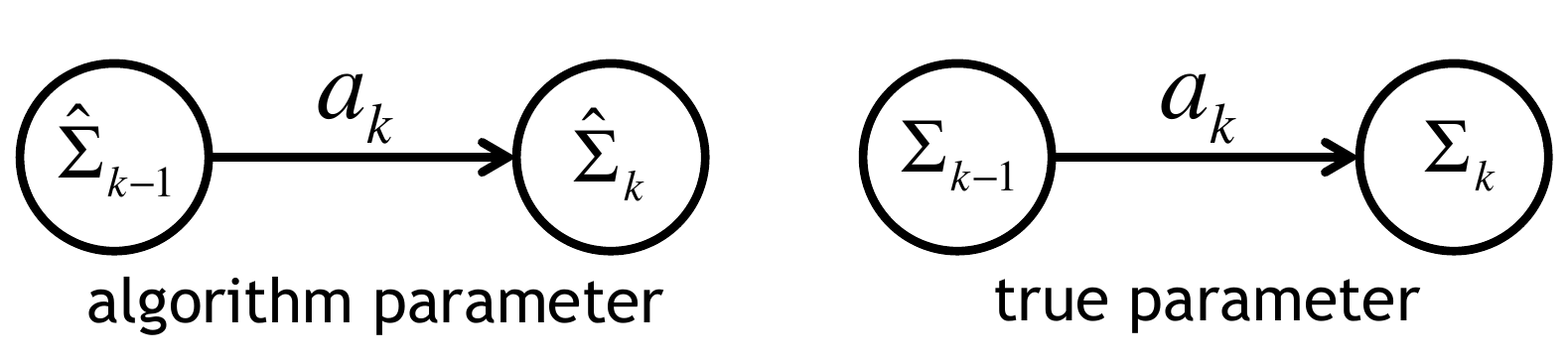}
\caption{Parameter update in the algorithm and for the true distribution.}
\label{evol}
\end{center}\vspace{-0.2in}

\end{figure}

\begin{algorithm}[h!]
  \caption{Info-Greedy Sensing for Gaussian signals}
  \begin{algorithmic}[1]
    \REQUIRE assumed signal mean \(\theta\) and covariance matrix \(\Gamma\), noise variance \(\sigma^2\), recovery accuracy \(\varepsilon\),
      confidence level \(p\)
    \REPEAT 
                         \STATE \(\lambda \leftarrow \Vert{\Gamma}\Vert\)
                          \STATE \(\beta \leftarrow (\chi_n^2(p)/\varepsilon^2-1/\lambda)\sigma^2\)

        \COMMENT{largest eigenvalue}
       \STATE \(u \leftarrow\)
        normalized eigenvector of \(\Gamma\)
        for eigenvalue \(\lambda\)
      
      \STATE form measurement: $a = \sqrt{\beta} u$
        \STATE measure: \(y = \transpose a x + w\)

     \STATE update mean: \(\theta \leftarrow \theta
        + \Gamma a
        (y - \transpose a\theta )/(\lambda +\sigma^2)\)
        
\STATE  update covariance: \(\Gamma \leftarrow \Gamma
        - \Gamma a
        \transpose a\Gamma/(\lambda + \sigma^2) \) 
        
          \UNTIL{\(\Vert{\Gamma}\Vert \leq \varepsilon^{2} / \chi_{n}^{2}(p)\)}
    \COMMENT{all eigenvalues become small}
    \RETURN posterior mean \(\theta\) as a signal estimate $\hat{x}$
  \end{algorithmic}
  \label{alg:Gaussian-all}
\end{algorithm}

\section{Performance bounds}

We analyze the performance of Info-Greedy Sensing, when the assumed covariance matrix is used for measurement design, $\widehat{\Sigma}$, which is different from the true signal covariance matrix $\Sigma$, i.e. $\widehat{\Sigma}$ is used to initialize Algorithm \ref{alg:Gaussian-all}.
Let the eigenpairs of $\Sigma$ with the eigenvalues (which can be zero) ranked from the largest to the smallest to be $(\lambda_1, u_1), (\lambda_2, u_2), \ldots, (\lambda_n, u_n)$, and let the eigenpairs of $\widehat{\Sigma}$  with the eigenvalues  (which can be zero) ranked from the largest to the smallest to be $(\hat{\lambda}_1, \hat{u}_1), (\hat{\lambda}_2, \hat{u}_2), \ldots, (\hat{\lambda}_n, \hat{u}_n)$. Let the updated covariance matrix in Algorithm \ref{alg:Gaussian-all} starting from $\widehat{\Sigma}$ after $k$ measurements using $\{a_i\}_{i=1}^k$  be $\widehat{\Sigma}_k$,  and the true conditional covariance matrix of the signal after these measurements be $\Sigma_k$. The evolution of the covariance matrices in Algorithm \ref{alg:Gaussian-all} is illustrated in Fig. \ref{evol}. 
Hence, by this notation, since each time we measure in direction of the dominating eigenvector of the updated covariance matrix, we have that $(\hat{\lambda}_k, \hat{u}_k)$ is the largest eigenpair of $\widehat{\Sigma}_{k-1}$, and that  $(\lambda_k, u_k)$ is the largest eigenpair of $\Sigma_{k-1}$.
Furthermore, denote the difference between the true and the assumed conditional covariance matrices after we obtain $k$ measurements 
\[E_k=\widehat{\Sigma}_k-\Sigma_k,\]
and let 
\[\delta_k=\Vert E_k\Vert.\] 
Assume the eigenvalues of $E_k$ are $e_1\geq e_2 \geq\cdots\geq e_n$. Then $\delta_k=\max\{\vert e_1 \vert, \vert e_n \vert\}$.

\subsection{Deterministic error}

The following theorem shows that when the error, $\|\widehat{\Sigma} - \Sigma\|$ is sufficiently small, the performance of Info-Greedy Sensing will not degrade much. Note that, however, if the power allocations $\beta_i$ are calculated using the eigenvalues of the assumed covariance matrix $\widehat{\Sigma}$, after $K = s$ iterations, we do not necessarily reach the desired precision $\varepsilon$ with probability $p$. 

\begin{theorem}\label{entropy}
Assume the power allocations $\beta_k = (\chi_n^2(p) /\varepsilon^2 - 1/\hat{\lambda}_k)\sigma^2$ are calculated using eigenvalues $\hat{\lambda}_k$ of $\widehat{\Sigma}$, the noise variance $\sigma^2$, recovery accuracy $\varepsilon$ and confidence level $p$ in Algorithm \ref{alg:Gaussian-all}.  Given the rank of the covariance matrix $\mbox{rank}(\Sigma) = s$, the number of total measurements is $K$, for some constant $0<\zeta<1$, if the error  satisfies 
\[
\|\Sigma - \widehat{\Sigma}\|\leq \frac{\zeta}{4^{K+1}}\frac{\varepsilon^2}{\chi_n^2(p)},\]
then 
\begin{equation}
\entropy[y_{j}, a_{j}, j \leq k]{x} \leq \frac{s}{2}
\left\{\ln [2\pi e {\rm tr}(\Sigma)]
- \sum_{j=1}^{k}\ln (1/f_j)
\right\}, \label{entropy_bound}
\end{equation}
where
\[
f_k=1-\frac{1-\zeta}{s}\frac{\beta_k\hat{\lambda}_k}{\beta_k\hat{\lambda}_{k}+\sigma^2} \in (0, 1),\quad k = 1, \ldots, K.
\]
\end{theorem}

In the proof of Theorem \ref{entropy}, we use the trace of the underlying actual covariance matrix ${\rm tr}(\Sigma_k)$ as potential function, which serves as a surrogate for the product of eigenvalues that determines the entropy, since the calculation of the trace of the observed covariance matrix ${\rm tr} (\widehat{\Sigma}_{k})$ is much easier. Note that for an assumed covariance matrix $\Sigma$, after measuring in the direction of a unit norm eigenvector $u$ with eigenvalue \(\lambda\)  using power \(\beta\),
the updated matrix takes the form of
\begin{equation}
  \label{eq:covariance-eigenvector}
  \begin{split}
  &\Sigma - \Sigma \sqrt{\beta} u
  \left(
    \transpose{\sqrt{\beta} u} \Sigma \sqrt{\beta} u
    + \sigma^{2}
  \right)^{-1}
  \transpose{\sqrt{\beta} u} \Sigma \\
&  =
  \frac{\lambda \sigma^{2}}{\beta \lambda + \sigma^{2}}
  u \transpose{u}
  + \Sigma^{\perp u},
  \end{split} 
\end{equation}
where \(\Sigma^{\perp a}\) is the component of \(\Sigma\)
in the orthogonal complement of \(a\).
Thus, the only change in the eigen-decomposition of \(\Sigma\)
is the update of the eigenvalue of \(a\)
from \(\lambda\) to
\(\lambda \sigma^{2} / (\beta \lambda + \sigma^{2})\).
Based on the update above in (\ref{eq:covariance-eigenvector}), after one measurement, the trace of the covariance matrix that the algorithm keeps track of becomes 
\[{\rm tr}(\widehat{\Sigma}_{k})={\rm tr}(\widehat{\Sigma}_{k-1})-\frac{\beta_{k}\hat{\lambda}_{k}^2}{\beta_{k}\hat{\lambda}_{k}+\sigma^2}.\]

\begin{remark}
The upper bound of the posterior signal entropy in (\ref{entropy_bound}) shows that the amount of uncertainty reduction by the $k$th measurement is roughly $(s/2) \ln (1/f_k)$.
\end{remark}
\begin{remark}
Use the inequality that $\ln(1-x) \leq -x$ for $x\in (0, 1)$, we have that in (\ref{entropy_bound})
\begin{align*}
\entropy[y_{j}, a_{j}, j \leq k]{x}
&\leq  \frac{s}{2}\ln[2\pi e{\rm tr}(\Sigma)]-\frac{1-\zeta}{2}\sum_{j=1}^k \frac{\beta_k\hat{\lambda}_k}{\beta_{k}\hat{\lambda}_k+\sigma^2}\\
&=\frac{s}{2}\ln[2\pi e{\rm tr}(\Sigma)]-\frac{k(1-\zeta)}{2}\\
&~~~+\frac{(1-\zeta)\varepsilon^2}{2\chi_n^2(p)}\sum_{j=1}^k\frac{1}{\hat{\lambda}_j}.
\end{align*}
On the other hand, if the true covariance matrix is used, the posterior entropy of the signal is given by 
\begin{align}
\mathbb{H}_{\rm ideal}\left[x,\middle| y_{j}, a_{j}, j \leq k\right]
&=\frac{1}{2} \ln [(2\pi e)^s \prod_{i=1}^s \lambda_i ]
-\frac{\chi^2 _n(p)}{2\varepsilon^2}\sum_{j=1}^k \lambda_i
\label{entropy_mismatch}
\end{align}
where $\tilde{\beta}_j = (\chi_n^2(p)/\varepsilon^2-1/\lambda_j)\sigma^2$. Hence, we have
\begin{align}
&\entropy[y_{j}, a_{j}, j \leq k]{x} \leq \nonumber\\
& 
\mathbb{H}_{\rm ideal}\left[x,\middle| y_{j}, a_{j}, j \leq k\right]
+ \frac{s}{2} \ln\frac{{\rm tr}(\Sigma)}{\sqrt[s]{\prod_{j=1}^s\lambda_i}} \nonumber \\
& ~~~- \frac{1}{2} \sum_{j=1}^k [
\frac{\chi^2 _n(p)}{\varepsilon^2}\lambda_i
 + (1-\zeta)(1-\frac{\varepsilon^2}{\chi_n^2(p))}\frac{1}{\hat{\lambda}_j}
].
\label{entropy_mismatch_2}
\end{align}
This upper bound has a nice interpretation: it characterizes the amount of uncertainty reduction with each measurement. For example, when the number of measurements required when using the assumed covariance matrix versus using the true covariance matrix are the same, we have $\lambda_i \geq \varepsilon^2/\chi^2_n(p)$ and $\hat{\lambda}_i \geq \varepsilon^2/\chi^2_n(p)$. Hence, the third term in (\ref{entropy_mismatch_2}) is upper bounded by $- k/2$, which means that the amount of reduction in entropy is roughly 1/2 nat per measurement. 
\end{remark}

\begin{remark}
Consider the special case where the errors only occur in the eigenvalues of the matrix but not in the eigenspace $U$, i.e. 
\[\widehat{\Sigma} - \Sigma = U \mbox{diag}\{e_1, \cdots, e_s\} \transpose U\] 
and $\max_{1\leq i\leq s}{\vert e_i \vert}=\delta_0$, 
the upper bound in (\ref{entropy_mismatch}) can be further simplified. 
Suppose only the first $K (K\leq s)$ largest eigenvalues of $\widehat{\Sigma}$ are larger than the stopping criterion $\varepsilon^2/\chi_n^2(p)$ required by the precision, i.e., the algorithm takes $K$ steps in total. Then  
\begin{align*}
\entropy[y_{j}, a_{j}, j \leq k]{x}
&\leq \mathbb{H}_{\rm ideal}\left[x,\middle| y_{j}, a_{j}, j \leq k\right] \\
&~~+K\ln(1+\frac{\chi_n^2(p)}{\varepsilon^2}\delta_K)\\
&~~+\sum_{j=K+1}^{s} \ln(1+\frac{\delta_0+\delta_K}{\lambda_j}). 
\end{align*}
This characterizes the gap between the signal posterior entropy using the correct versus the incorrect covariance matrices after all measurements have been used. 

\end{remark}

If we allow more total power and use a different power allocation scheme than what is prescribed in Algorithm \ref{alg:Gaussian-all}, we are able to reach the desired precision $\varepsilon$. The following theorem establishes an upper bound on the amount of extra total power needed to reach the same precision $\varepsilon$ (than the total power $P_{\rm ideal}$ if using the correct covariance matrix).

\begin{theorem}\label{thm:power}
Given the recovery precision $\varepsilon$, confidence level $p$, rank of the true covariance matrix ${\rm rank}(\Sigma)=s$, assume $K\leq s$ eigenvalues of $\Sigma$ are larger than $\varepsilon^2/\chi_n^2(p)$. If 
\[\|\widehat{\Sigma} - \Sigma\|\leq \frac{1}{4^{s+1}}\frac{\varepsilon^2}{\chi_n^2(p)},\] 

then to reach a precision $\varepsilon$ at confidence level $p$, the total power $P_{\rm mismatch}$ required by Algorithm \ref{alg:Gaussian-all} when using $\widehat{\Sigma}$ is upper bounded by
\[P_{\rm mismatch} < P_{\rm ideal} + [\frac{20}{51}s+\frac{1}{272}K]\frac{\chi_n^2(p)}{\varepsilon^2}\sigma^2.\]

\end{theorem}

\begin{remark}
In a special case when $K= s$ eigenvalues of $\Sigma$ are larger than $\varepsilon^2/\chi_n^2(p)$, 
then under the condition of Theorem \ref{thm:power}, we have a simpler expression for the upper bound
\begin{align*}
P_{\rm mismatch}
& < P_{\rm ideal} + \frac{323}{816}\frac{\chi_n^2(p)}{\varepsilon^2} \sigma^2 s.
\end{align*}
Note that the additional power required is only linear in $s$, which is quite small. All other parameters are independent of the input matrix.
\end{remark}

Also, note that when there is a mismatch in the assumed covariance matrix, better performance can be achieved if we make many low power measurements than making one full power measurement because we update the assumed covariance matrix in between. 

\subsection{Initialization with sample covariance matrix}

In practice, we usually use a sample covariance matrix for $\widehat{\Sigma}$. When the samples are Gaussian distributed, the sample covariance matrix follows a Wishart distribution. By finding the tail probability of the Wishart distribution, we are able to establish a lower bound on the number of samples to form the sample covariance matrix so that the conditions required by Theorem \ref{entropy} are met with high probability and, hence, Algorithm \ref{alg:Gaussian-all} has good performance with the assumed matrix $\widehat{\Sigma}$.

\begin{corollary}\label{cor:sampleSize}
Suppose the sample covariance matrix is obtained from training samples $\tilde{x}_1,\ldots,\tilde{x}_L$ that are drawn i.i.d. from  $\mathcal N(0,\Sigma)$:  $$\widehat{\Sigma}=\frac{1}{L}\sum_{i=1}^{L}\tilde{x}_i \tilde{x}_i^{\intercal}.$$
Let $\delta_0 = \|\widehat{\Sigma} - \Sigma\|$.
When
\[L\geq 4n^{1/2}{\rm tr}(\Sigma)(\frac{\Vert\Sigma\Vert}{\delta_0^2}+\frac{4}{\delta_0}),\]
we have $\Vert\widehat{\Sigma} -\Sigma\Vert \leq \delta_0$ with probability exceeding $1-2n\exp(-\sqrt{n})$.
\end{corollary}

\subsection{Initialization with covariance sketching}

We may also use a covariance sketching scheme to form an estimate of the covariance matrix to initialize the algorithm, as illustrated in Fig. \ref{cov_sketch}. Covariance sketching is based on sketches $\gamma_j$, $j = 1, \ldots, M$, of the samples $\tilde{x}_i$, $i = 1, \ldots, N$ drawn from the signal distribution. The sketches are formed by linearly projecting these samples via random sketching vectors $b_{i}$, $i = 1, \ldots, M$ and then computing the average energy over $L$ repetitions. The sketching can be shown to be a linear operator $\mathcal{B}$ applied on the original covariance matrix $\Sigma$, as demonstrated in Appendix \ref{app:cov_sketch}. 
Then we may recover the original covariance matrix from these sketches $\gamma$ by solving the following convex program 
\begin{equation}\label{opt}
\begin{array}{rl}
\widehat{\Sigma}= \argmin_{X} & {\rm tr}(X)\\
{\rm subject\ to}& X\ \succeq 0,\ \Vert \gamma-\mathcal B(X)\Vert_1\leq \tau,
\end{array}
\end{equation}
where $\tau$ is a user parameter that specifies the noise level.

\begin{figure}

\begin{center}
\includegraphics[width = 0.45\linewidth]{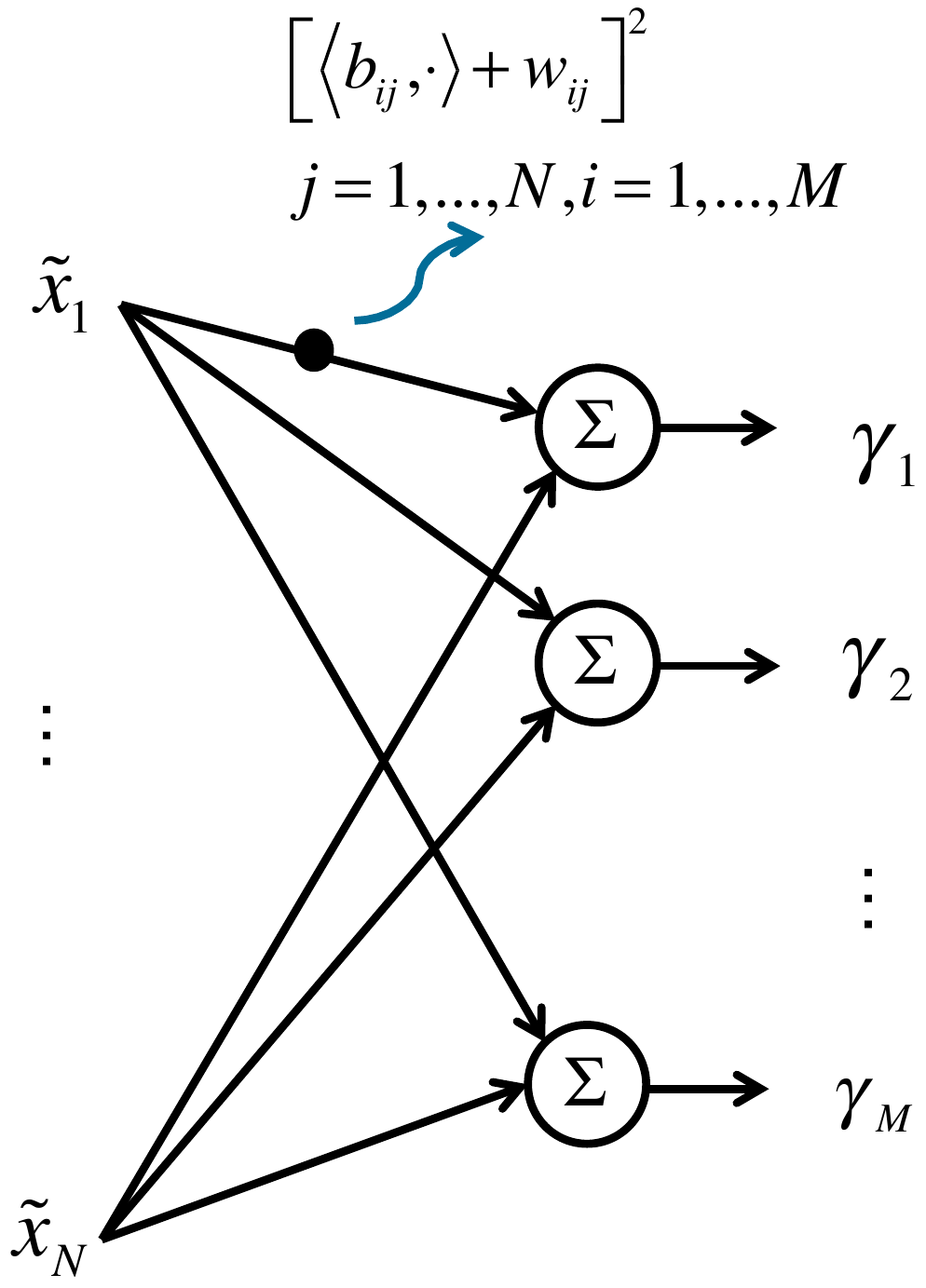}
\caption{Diagram of covariance sketching in our setting. The circle aggregates quadratic sketches from branches and computes the average. }
\label{cov_sketch}
\end{center}\vspace{-0.2in}
\end{figure}

We further establish conditions on the covariance sketching so that such an initialization for Info-Greedy Sensing is sufficient.

\begin{theorem}\label{thm:cov-sketch}
Assume 
the setup of covariance sketching as above. Then with probability exceeding
$1-2/n-{2}/{\sqrt{n}}-2n\exp(-\sqrt{n}))-\exp(-c_0c_1ns)$,  the solution to (\ref{opt}) satisfies 
\[\Vert \widehat{\Sigma}-\Sigma \Vert\leq \delta_0,\] for some $\delta_0 > 0$, as long as 
for some constant $c > 0$ the parameters $M$, $N$, $L$, and $\tau$ are chosen such that
\begin{align*}
M&\triangleq cns>c_0 ns,\\
N 
&\geq 4n^{1/2}{\rm tr}(\Sigma)(\frac{36c^2n^4s^2\Vert\Sigma\Vert}{\tau^2}+\frac{24cn^2s}{\tau}),\\
L
&\geq \max\{ \frac{cs}{4n\Vert\Sigma\Vert}\sigma^2, \ \frac{1}{\sqrt{2{\rm tr}{(\Sigma)}\Vert\Sigma\Vert csn^3}}\sigma^2, \frac{6cns}{\tau}\sigma^2\},\\
\tau& =cns\delta_0/C_2.
\end{align*}
Here $c_0$, $c_1$, $C_1$, and $C_2$ are absolute constants.

\end{theorem}

\section{Numerical example}

When the assumed covariance matrix for the signal $x$ is equal to its true covariance matrix, Info-Greedy Sensing is identical to the batch method \cite{CarsonChenRodrigues2012} (the batch method measures using the largest eigenvectors of the signal covariance matrix). However, when there is a mismatch between the two, Info-Greedy Sensing outperforms the batch method due to its adaptivity, as shown by the example demonstrated in Fig. \ref{Fig:mismatch}. Info-Greedy Sensing also outperforms the sensing algorithm where $a_i$ are chosen to be random Gaussian vectors with the same power allocation, as it uses prior knowledge (albeit being imprecise) about the signal distribution.

\begin{figure}
\begin{center}
\includegraphics[width = 0.8\linewidth]{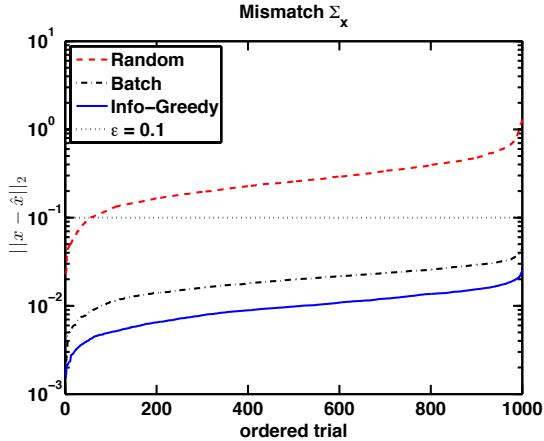}
\end{center}
\caption{Sensing a low-rank Gaussian signal of dimension $n= 500$ and about $5\%$ of the eigenvalues are non-zero, when there is mismatch between the assumed covariance matrix and true covariance matrix: ${\Sigma_x}_{\rm , assumed} = {\Sigma_x}_{\rm , true} + e \transpose{e}$, where $e\sim \mathcal{N}(0, I)$, and using 20 measurements. The batch method measures using the largest eigenvectors of ${\Sigma_x}_{\rm , assumed}$, and the Info-Greedy Sensing updates ${\Sigma_x}_{\rm , assumed}$ in the algorithm. Info-Greedy Sensing is more robust to mismatch than the batch method.}
\label{Fig:mismatch}
\end{figure}

\section{Discussion}

In high-dimensional problems, a commonly used low-dimensional signal model
for $x$ is to assume the signal lies in a
 subspace plus Gaussian noise, which corresponds to the case we considered in this paper where the signal covariance is low-rank. A more general model is the Gaussian mixture model (GMM), which can be viewed as a model for the signal lying in a
  union of multiple subspaces plus Gaussian noise, and it has been
  widely used in image and video analysis among
  others. 
  Our analysis for a low-rank Gaussian signal can be easily extended to an analysis of a low-rank Gaussian mixture model (GMM). Such results for GMM are quite general and can be used for an arbitrary signal distribution. 
In fact,
parameterizing via low-rank GMMs is a popular way to approximate
complex densities for high-dimensional data. Hence, we may be able to
couple the results for Info-Greedy Sensing of GMM with the
recently developed methods of scalable multi-scale density estimation
based on empirical Bayes \cite{WangCanale2014} to create powerful
tools for information guided sensing for a general signal model. We
may also be able to obtain performance guarantees using multiplicative weight update techniques
together with the error bounds in \cite{WangCanale2014}.

\section*{Acknowledgement}

This work is partially supported by an NSF CAREER Award CMMI-1452463
and an NSF grant CCF-1442635. Ruiyang Song was visiting the H. Milton Stewart School of Industrial and Systems Engineering at the Georgia Institute of Technology while working on this paper.

\bibliographystyle{ieeetr}
\bibliography{bib}

\begin{thebibliography}{10}

\bibitem{InfoGreedy2014}
G.~Braun, S.~Pokutta, and Y.~Xie, ``Info-greedy sequential adaptive compressed
  sensing,'' {\em to appear in IEEE J. Sel. Top. Sig. Proc.}, 2014.

\bibitem{NeifeldTaskSpecific2008}
A.~Ashok, P.~Baheti, and M.~A. Neifeld, ``Compressive imaging system design
  using task-specific information,'' {\em Applied Optics}, vol.~47, no.~25,
  pp.~4457--4471, 2008.

\bibitem{KeAshok2010}
J.~Ke, A.~Ashok, and M.~Neifeld, ``Object reconstruction from adaptive
  compressive measurements in feature-specific imaging,'' {\em Applied Optics},
  vol.~49, no.~34, pp.~27--39, 2010.

\bibitem{NeifeldCSImaging2011}
A.~Ashok and M.~A. Neifeld, ``Compressive imaging: hybrid measurement basis
  design,'' {\em J. Opt. Soc. Am. A}, vol.~28, no.~6, pp.~1041-- 1050, 2011.

\bibitem{WirelessHouseElectricity2014}
W.~Boonsong and W.~Ismail, ``Wireless monitoring of household electrical power
  meter using embedded {RFID} with wireless sensor network platform,'' {\em
  Int. J. Distributed Sensor Networks, Article ID 876914, 10 pages}, vol.~2014,
  2014.

\bibitem{sparseSensorLocal2011}
B.~Zhang, X.~Cheng, N.~Zhang, Y.~Cui, Y.~Li, and Q.~Liang, ``Sparse target
  counting and localization in sensor networks based on compressive sensing,''
  in {\em IEEE Int. Conf. Computer Communications (INFOCOM)}, pp.~2255 -- 2258,
  2014.

\bibitem{HauptAdaptiveCS2009}
J.~Haupt, R.~Nowak, and R.~Castro, ``Adaptive sensing for sparse signal
  recovery,'' in {\em IEEE 13th Digital Signal Processing Workshop and 5th IEEE
  Signal Processing Education Workshop (DSP/SPE)}, pp.~702 -- 707, 2009.

\bibitem{DavenportArias-Castro2012}
M.~A. Davenport and E.~Arias-Castro, ``Compressive binary search,'' {\em
  arXiv:1202.0937v2}, 2012.

\bibitem{TajerPoor2012}
A.~Tajer and H.~V. Poor, ``Quick search for rare events,'' {\em
  arXiv:1210:2406v1}, 2012.

\bibitem{MalioutovSanghaviWillsky2010}
D.~Malioutov, S.~Sanghavi, and A.~Willsky, ``Sequential compressed sensing,''
  {\em IEEE J. Sel. Topics Sig. Proc.}, vol.~4, pp.~435--444, April 2010.

\bibitem{HauptSeqCS2012}
J.~Haupt, R.~Baraniuk, R.~Castro, and R.~Nowak, ``Sequentially designed
  compressed sensing,'' in {\em Proc. IEEE/SP Workshop on Statistical Signal
  Processing}, 2012.

\bibitem{JainSoniHaupt2013}
S.~Jain, A.~Soni, and J.~Haupt, ``Compressive measurement designs for
  estimating structured signals in structured clutter: A {B}ayesian
  experimental design approach,'' {\em arXiv:1311.5599v1}, 2013.

\bibitem{MalloyNowak2013}
M.~L. Malloy and R.~Nowak, ``Near-optimal adaptive compressed sensing,'' {\em
  arXiv:1306.6239v1}, 2013.

\bibitem{KrishnamurthySingh13}
A.~Krishnamurthy, J.~Sharpnack, and A.~Singh, ``Recovering graph-structured
  activations using adaptive compressive measurements,'' in {\em Annual
  Asilomar Conference on Signals, Systems, and Computers}, Sept. 2013.

\bibitem{ErvinCastro2013}
T.~Ervin and R.~Castro, ``Adaptive sensing for estimation of structure sparse
  signals,'' {\em arXiv:1311.7118}, 2013.

\bibitem{AkshayHaupt2014}
S.~Akshay and J.~Haupt, ``On the fundamental limits of recovering tree sparse
  vectors from noisy linear measurements,'' {\em IEEE Trans. Info. Theory},
  vol.~60, no.~1, pp.~133--149, 2014.

\bibitem{BayesianCS2008}
S.~Ji, Y.~Xue, and L.~Carin, ``Bayesian compressive sensing,'' {\em IEEE Trans.
  Sig. Proc.}, vol.~56, no.~6, pp.~2346--2356, 2008.

\bibitem{TaskDrivenDuarte2013}
J.~M. Duarte-Carvajalino, G.~Yu, L.~Carin, and G.~Sapiro, ``Task-driven
  adaptive statistical compressive sensing of {Gaussian} mixture models,'' {\em
  IEEE Trans. Sig. Proc.}, vol.~61, no.~3, pp.~585--600, 2013.

\bibitem{CarsonChenRodrigues2012}
W.~Carson, M.~Chen, R.~Calderbank, and L.~Carin, ``Communication inspired
  projection design with application to compressive sensing,'' {\em SIAM J.
  Imaging Sciences}, 2012.

\bibitem{WangCanale2014}
Y.~Wang, A.~Canale, and D.~Dunson, ``Scalable multiscale density estimation,''
  {\em arXiv:1410.7692}, 2014.

\bibitem{stewart1990matrix}
G.~W. Stewart and J.-G. Sun, {\em Matrix perturbation theory}.
\newblock Academic Press, Inc., 1990.

\bibitem{ChenChiGoldsmith2014}
Y.~Chen, Y.~Chi, and A.~J. Goldsmith, ``Exact and stable covariance estimation
  from quadratic sampling via convex programming,'' {\em IEEE Trans. Info.
  Theory}, vol.~in revision, 2013.

\bibitem{zhu2012short}
S.~Zhu, ``A short note on the tail bound of wishart distribution,'' {\em
  arXiv:1212.5860}, 2012.

\end{thebibliography}

\appendices

\section{Covariance sketching}\label{app:cov_sketch}

We consider the following setup for covariance sketching. Suppose we are able to form measurement in the form of $y = \transpose a x + w$ like we have in the Info-Greedy Sensing algorithm. 
Suppose there are $N$ copies of Gaussian signal we would like to sketch: $\tilde{x}_1,\ldots, \tilde{x}_N$ that are i.i.d. sampled from $\mathcal{N}(0, \Sigma)$, and we sketch using $M$ random vectors: $b_1, \ldots, b_M$. Then for each fixed sketching vector $b_i$, and fixed copy of the signal $\tilde{x}_j$, we acquire $L$ noisy realizations of the projection result $y_{ijl}$ via
\[y_{ijl}=\transpose b_i \tilde{x}_j +w_{ijl}, \quad l = 1, \ldots, L.\]
We choose the random sampling vectors $b_i$ as i.i.d. Gaussian with zero mean and covariance matrix equal to an identity matrix. Then we average $y_{ijl}$ over all realizations $l = 1, \ldots, L$ to form the $i$th sketch $y_{ij}$ for a single copy $\tilde{x}_j$: 
\[y_{ij}=\transpose b_i\tilde{x}_j +\underbrace{\frac{1}{L}\sum_{l=1}^L w_{ijl}}_{w_{ij}}.\]
The average is introduced to suppress measurement noise, which can be viewed as a generalization of sketching using just one sample. 
Denote $w_{ij}=\frac{1}{L}\sum_{l=1}^L w_{ijl}$, which is distributed as $\mathcal N(0, \sigma^2/L)$. Then we will use average energy of the sketches as our data $\gamma_i$, $i = 1, \ldots, M$, for covariance recovery:
\[
\gamma_i \triangleq \frac{1}{N}\sum_{j=1}^{N}y_{ij}^2.
\]
Note that $\gamma_i$ can be further expanded as 
\begin{equation}
\gamma_i = {\rm tr}(\widehat{\Sigma}_N b_i \transpose b_i)+\frac{2}{N}\sum_{j=1}^N w_{ij}\transpose b_i \tilde{x}_j +\frac{1}{N}\sum_{j=1}^N w_{ij}^2, \label{meas}
\end{equation}
where \[\widehat{\Sigma}_N=\frac{1}{N}\sum_{j=1}^{N}\tilde{x}_j \tilde{x}^{\intercal}_j,\] is the maximum likelihood estimate of $\Sigma$ (and is also unbiased). We can write (\ref{meas}) in vector matrix notation as follows. 
Let
$\gamma=[\gamma_1,\cdots \gamma_M\transpose]$. Define a linear operator 
 $\mathcal B:\mathbb R^{n\times n}\mapsto \mathbb R^{M}$ such that  $\mathcal [B(X)]_i={\rm tr}(X b_i \transpose b_i)$. Thus, we can write (\ref{meas}) as a linear measurement of the true covariance matrix $\Sigma$  
 \[\gamma=\mathcal{B} (\Sigma)+\eta,\]
where $\eta \in \mathbb{R}^M$ contains all the error terms and corresponds to the noise in our covariance sketching measurements, with the $i$th entry given by $$\eta_i=\transpose b_i(\widehat{\Sigma}_N-\Sigma) b_i+\frac{2}{N}\sum_{j=1}^N w_{ij}\transpose b_i \tilde{x}_j +\frac{1}{N}\sum_{j=1}^N w_{ij}^2.$$
Note that we can further bound the $\ell_1$ norm of the error term as
\begin{align*}
\Vert\eta\Vert_1 =\sum_{i=1}^M\vert\eta_i\vert
\leq \Vert\widehat{\Sigma}_N-\Sigma\Vert b+ 2\sum_{i=1}^M\vert z_i\vert+w,
\end{align*}
where $$b=\sum_{i=1}^M \Vert b_i\Vert^2,\ \mathbb E[b]=Mn,\ {\rm Var}[b] =2Mn,$$
$$w=\frac{1}{N}\sum_{i=1}^M \sum_{j=1}^N w_{ij}^2,\ \mathbb E[w]=M\sigma^2/L,\ {\rm and}\ {\rm Var} [w]=\frac{2M\sigma^4}{NL^2},$$
$$z_i=\frac{1}{N}\sum_{j=1}^N w_{ij}\transpose b_i \tilde{x}_j,\ \mathbb E[z_i]=0\ {\rm and}\ {\rm Var} [z_i]=\frac{\sigma^2 {\rm tr}(\Sigma)} {NL}.$$

We may recover the true covariance matrix from the sketches $\gamma$ using the convex optimization problem (\ref{opt}).

\section{Backgrounds}

\begin{lemma}{\cite{stewart1990matrix}}\label{lemmaeig}
Let $\Sigma$, $\widehat{\Sigma}\in\mathbb{R}^{p\times p}$ be symmetric,with eigenvalues $\lambda_1\geq\cdots\geq \lambda_p$ and $\hat{\lambda}_1\geq\cdots\geq\hat{\lambda}_p $ respectively. $E=\widehat{\Sigma}-\Sigma$ has eigenvalues $e_1\geq\cdots\geq e_p$. Then for each $i\in\{1,\cdots,p\}$,
\[\hat{\lambda}_i\in[\lambda_i+e_p, \lambda_i+e_1].\]
\end{lemma}


\begin{lemma}{\cite{ChenChiGoldsmith2014}}
\label{sketching}
Denote $\mathcal A:\mathbb R^{n\times n}\mapsto \mathbb R^{m}$ a linear operator and for $X\in \mathbb R^{n\times n}$, $\mathcal A(X)=\{a_i^T X a_i\}_{i=1}^{m}$. Suppose the measurement is contaminated by noise $\eta\in R^m$, i.e. $Y=\mathcal A(\Sigma)+\eta$ and assume $\Vert \eta \Vert_1\leq \epsilon_1$. Then with probability exceeding $1-\exp (-c_1 m)$ the solution $\widehat{\Sigma}$ to the trace minimization (\ref{opt}) satisfies
\[\Vert \widehat{\Sigma}-\Sigma \Vert_F\leq C_1\frac{\Vert\Sigma-\Sigma_r\Vert_{*}}{\sqrt{r}}+C_2\frac{\epsilon_1}{m},\]
for all $\Sigma \in R^{n\times n}$, provided that $m> c_0nr $. $c_0$, $c_1$, $C_1$ and $C_2$ are absolute constants and $\Sigma_r$ represents the best rank-r approximation of $\Sigma$.
When $\Sigma_r$ is exactly rank-$r$
$$\Vert \widehat{\Sigma}-\Sigma \Vert_F\leq C_2\frac{\epsilon_1}{m}.$$
\end{lemma}
\begin{lemma}{\cite{zhu2012short}}\label{zhu2012}
\label{Tail}
If $X\in \mathbb R^{n\times n}\sim \mathcal W_n(N,\Sigma)$, then for $t>0$,
$$P\{\Vert\frac{1}{N}X-\Sigma\Vert\geq (\sqrt{\frac{2t(\theta+1)}{N}}+\frac{2t\theta}{N})\Vert\Sigma \Vert\}\leq 2n\exp(-t),$$
where $\theta=\rm tr(\Sigma)/\Vert \Sigma \Vert.$
\end{lemma}

\section{Proofs}

\begin{lemma}\label{propdelta}

Suppose the power of measurement in the $k$th step is $\beta_{k}$.
If $\delta_{k-1}\leq{3\sigma^2}/{4\beta_{k}}$, $\delta_{k}\leq 4\delta_{k-1}$.

\end{lemma}

\begin{proof}
Let $\widehat{A}_{k}={a}_{k}{a}^{\intercal}_{k}$, and $\Vert \widehat{A}_{k} \Vert=\beta_{k}$,
\begin{align*}
E_{k}
&=E_{k-1}+\frac{\Sigma_{k-1} {a}_{k} {a}_{k}^{\intercal}\Sigma_{k-1}}{{a}_{k-1}^{\intercal}\Sigma_{k-1} {a}_{k}+\sigma^2}-\frac{\hat{\lambda}_{k}{a}_{k}{a}_{k}^{\intercal}\hat{\Sigma}_{k-1}}{\beta_{k}\hat{\lambda}_{k}+\sigma^2}\\
\delta_{k}
&\leq \delta_{k-1} +\frac{\beta_{k}\hat{\lambda}_{k}{a}_{k} E_{k-1}{a}_{k}}{(\beta_{k}\hat{\lambda}_{k}+\sigma^2)
(\beta_{k}\hat{{\lambda}}_{k}+\sigma^2-{a}_{k}^{\intercal} E_{k-1}{a})}\cdot \Vert \widehat{A}_{k}\widehat{\Sigma}_{k-1}\Vert\\
&~~+\frac{1}{\beta_{k}\hat{{\lambda}}_{k}+\sigma^2-{a}_{k}^{\intercal} E_{k-1}{a}_{k}}\\
&~~\cdot[\hat{\lambda}_{k}(\Vert \widehat{A}_{k}E_{k-1}\Vert]
+\Vert E_{k-1}\widehat{A}_{k}\Vert)
+\Vert E_{k-1}\widehat{A}_{k}E_{k-1}\Vert]\\
&\leq \delta_{k-1} +\frac{\beta_{k}^2\hat{\lambda}_{k}^2\delta_{k-1}}{(\beta_{k}\hat{\lambda}_{k}+\sigma^2)(\beta_{k}\hat{{\lambda}}_{k}+\sigma^2-\beta_{k}\delta_{k-1})}\\
&~~+\frac{\beta_{k}}{\beta_{k}\hat{\lambda}_{k}+\sigma^2-\beta_{k}\delta_{k-1}}[2\hat{\lambda}_{k}\delta_{k-1}+\delta_{k-1}^2]\\
&\leq (1+3\frac{\beta_{k}\hat{\lambda}_{k}}{ \beta_{k}\hat{\lambda}_{k}+\sigma^2-\beta_{k}\delta_{k-1} })\delta_{k-1}\\
&~~+\frac{\beta_{k}}{\beta_{k}\hat{\lambda}_{k}+\sigma^2-\beta_{k}\delta_{k-1}}\delta_{k-1}^2.
\end{align*}
Now that $\delta_{k-1} \leq \frac{3\sigma^2}{4\beta_{k}}$, we have 
$\delta_{k} \leq 4\delta_{k-1}$. 

\end{proof}
\begin{lemma}\label{proprank}
Consider positive semi-definite matrix $X \in \mathbb R^{n\times n}$, for $h\in \mathbb R^n$, if 
\[Y=X-\frac{1}{h^{\intercal}X h+\sigma^2}Xhh^{\intercal}X,\] we have \[{\rm rank }(X)={\rm rank}(Y).\]
\end{lemma}

\begin{proof}

Apparently, $\forall x\in {\rm ker}(X)$, $Yx=0$, i.e. \[{\ker}(X)\subset{ \ker }(Y).\]
Apply a decomposition for the positive semi-definite matrix $X=Q^{\intercal }Q$.
For $\forall x\in \ker (Y)$, let $b=Qh$, $z=Qx$. 
If $b=0$, $Y=X$; otherwise, when $b\neq 0$, 
we have \[0=x^{\intercal}Yx=z^{\intercal}z-\frac{z^{\intercal}bb^{\intercal}z}{b^{\intercal}b +\sigma^2}.\]
Thus,
\[ z^{\intercal} z=\frac{z^{\intercal}bb^{\intercal}z}{b^{\intercal}b +\sigma^2}\leq \frac{b^{\intercal} b}{ b^{\intercal} b+\sigma^2} z^{\intercal} z.\]
Therefore $z=0$, i.e. $x\in \ker (X)$,
$\ker (Y)\subset \ker (X).$
This shows that $\ker(X)=\ker(Y)$,
which leads to
${\rm rank}(X)={\rm rank} (Y).$
\end{proof}

\begin{lemma}\label{proptrace}
If $\delta_{k-1}\leq \hat{\lambda}_{k}$, the true conditional covariance matrix $\Sigma_k$ of the signal $x$ conditioned upon the  measurements $y_1, \ldots, y_k$ is related to the previous iteration as follows:
\begin{align*}
{\rm tr}(\Sigma_k)\leq {\rm tr}(\Sigma_{k-1})
&-\frac{\beta_{k}\hat{\lambda}_{k}^2}{\beta_{k}\hat{\lambda}_{k}+\sigma^2}+\frac{3\beta_{k}\hat{\lambda}_{k}\delta_{k-1}}{\beta_{k}\hat{\lambda}_{k}+\sigma^2-\beta_{k}\delta_{k-1}}.
\end{align*}
\end{lemma}
\begin{proof}
Let $\widehat{A}_{k}={a}_{k}{a}^{\intercal}_{k}$.
\begin{align*}
E_k&=E_{k-1} +\hat{\lambda}_{k}^2\widehat{A}_{k}\\
&\cdot\frac{{a}_{k}^{\intercal}E_{k-1}{a}_{k}}{(\beta_{k}\hat{\lambda}_{k}+\sigma^2)(\beta_{k}\hat{\lambda}_{k}+\sigma^2-{a}_{k}^{\intercal}E_{k-1}{a}_{k})}\\
&-\frac{\hat{\lambda}_{k}}{\beta_{k}\hat{\lambda}_{k}+\sigma^2-{a}_{k}^{\intercal}E_{k-1}{a}_{k}}
\cdot(\widehat{A}_{k}E_{k-1}+E_{k-1}\widehat{A}_{k})\\
&+\frac{1}{\beta_{k}\hat{\lambda}_{k}+\sigma^2-{a}_{k}^{\intercal}E_{k-1}{a}_{k}}E_{k-1}\widehat{A}_{k}E_{k-1}.
\end{align*}
Note that ${\rm rank}(\widehat{A}_{k})=1$, thus ${\rm rank}(\widehat{A}_{k}E_{k-1})\leq 1$, therefore it has at most one nonzero eigenvalue,
\begin{align*}
\vert{\rm tr}(\widehat{A}_{k}E_{k-1})\vert
&=\vert{\rm tr}({E_{k-1}\widehat{A}_{k}})\vert \\
&=\Vert\widehat{A}_{k}E_{k-1}\Vert
\leq\Vert\widehat{A}_{k}\Vert\Vert E_{k-1}\Vert =\beta_{k}\delta_{k-1}.
\end{align*}
Note that $E_{k-1}$ is symmetric and $\hat{A}_{k}$ is positive semi-definite, we have
$
{\rm tr}(E_{k-1}\widehat{A}_{k}E_{k-1})\geq 0.
$
Hence, 
\begin{align*}
\rm{tr}(E_k)&=\rm{tr}(\widehat{\Sigma}_{k})-\rm{tr}(\Sigma_k)\\
&\geq {\rm tr}(E_{k-1})-\frac{3\beta_{k}\hat{\lambda}_{k}(\beta_{k}\hat{\lambda}_{k}+\frac{2\sigma^2}{3})\delta_{k-1}}{(\beta_{k}\hat{\lambda}_{k}+\sigma^2)(\beta_{k}\hat{\lambda}_{k}+\sigma^2-\beta_{k}\delta_{k-1})}.
\end{align*}
Therefore,
\begin{align*}
{\rm tr}(\Sigma_k)\leq {\rm tr}(\Sigma_{k-1})
&-\frac{\beta_{k}\hat{\lambda}_{k}^2}{\beta_{k}\hat{\lambda}_{k}+\sigma^2}+\frac{3\beta_{k}\hat{\lambda}_{k}\delta_{k-1}}{\beta_{k}\hat{\lambda}_{k}+\sigma^2-\beta_{k}\delta_{k-1}}.
\end{align*}
\end{proof}
\begin{lemma}\label{est}
Denote $\theta={\rm tr}(\Sigma)/\Vert\Sigma \Vert\geq 1$, ${\rm rank}(\Sigma)=s$, $M=cns$,
if \[N \geq 4n^{1/2}{\rm tr}(\Sigma)(\frac{36c^2n^4s^2\Vert\Sigma\Vert}{\tau^2}+\frac{24cn^2s}{\tau})\]
and
\[L\geq \max\{ \frac{cs}{4n\Vert\Sigma\Vert}\sigma^2, \ \frac{1}{\sqrt{2{\rm tr}{(\Sigma)}\Vert\Sigma\Vert csn^3}}\sigma^2, \frac{6cns}{\tau}\sigma^2\},\]
then with probability exceeding $1-{2}/{n}-{2}/{\sqrt{n}}-2n\exp(-\sqrt{n})$
we have $\Vert\eta\Vert_1\leq \tau.$
\end{lemma}
\begin{proof}
From Chebyshev's inequality, we have that
\[P\{\vert z_i\vert<\frac{\tau}{6M}\} \geq 1-\frac{36M^2\sigma^2{\rm tr}(\Sigma)}{NL\tau^2},\] 
\[P\{w<M\frac{\sigma^2}{L}+\frac{\tau}{6}\} \geq 1-\frac{72\sigma^4M}{NL^2\tau^2},\] and 
\[P\{b<(M+\sqrt{M})n\} \geq 1-\frac{2}{n}.\] 
Let $\delta_{\Sigma}={\tau}/[{3n(M+\sqrt{M})}].$ When 
\begin{align*}
N 
& \geq 4n^{1/2}{\rm tr}(\Sigma)(\frac{36n^2M^2\Vert\Sigma\Vert}{\tau^2}+\frac{24nM}{\tau})
\end{align*}
with with Lemma \ref{zhu2012}, we have 
\begin{align*}
&P\{\Vert\widehat{\Sigma}_N-\Sigma\Vert \leq \delta_{\Sigma}\}\\
\geq & P\{\Vert\widehat{\Sigma}_N-\Sigma\Vert \leq (\sqrt{\frac{2n^{1/2}(\theta+1)}{N}}+\frac{2\theta n^{1/2}}{N})\Vert\Sigma\Vert\}\\
> & 1-2n\exp(-\sqrt{n}),
\end{align*}
when 
\[L\geq \max\{ \frac{cs}{4n\Vert\Sigma\Vert}\sigma^2, \ \frac{1}{\sqrt{2{\rm tr}{(\Sigma)}\Vert\Sigma\Vert csn^3}}\sigma^2, \frac{6cns}{\tau}\sigma^2\},\]
we have 
\begin{align*}
P\{\vert z_i\vert<\frac{\tau}{6M}\}
&\geq 1-\frac{1}{M\sqrt{n}},\\
P\{w<\frac{\tau}{3}\}
&\geq 1-\frac{1}{\sqrt{n}},\\
P\{\vert b\vert<(M+\sqrt{M})n\}
&\geq 1-\frac{2}{n}.
\end{align*}
Therefore, $\Vert\eta\Vert_1\leq \tau$
holds with probability at least $1-{2}/{n}-{2}/{\sqrt{n}}-2n\exp(-\sqrt{n})$.
\end{proof}
\begin{proof}[Proof of Theorem \ref{entropy}]
Recall that for $k=1,\ldots, K$, $\hat{\lambda}_k\geq \varepsilon^2/\chi_n^2(p)$. With Lemma \ref{propdelta} provided in the Appendix, we can show that for some constant $0<\zeta<1$, if $\delta_0\leq \zeta\varepsilon^2/(4^{K+1}\chi_n^2(p))$, for the first $K$ measurements,
\[\delta_k\leq\frac{1}{4^{K-k}}\frac{\zeta\varepsilon^2}{4\chi_n^2(p)}\leq\frac{1}{4^{K-k}}\frac{3\sigma^2}{4\beta_1},\ k=1,\ldots, K.\]
By applying the result in Lemma \ref{proptrace}, we have
\begin{align*}
{\rm tr}(\Sigma_k)
&\leq {\rm tr}{(\Sigma_{k-1})}-(1-\zeta)\frac{\beta_{k}\hat{\lambda}_{k}}{\beta_{k}\hat{\lambda}_{k}+\sigma^2}\lambda_{k}\\
&\leq {\rm tr}(\Sigma_{k-1})-(1-\zeta)\frac{\beta_{k}\hat{\lambda}_k}{\beta_{k}\hat{\lambda}_{k}+\sigma^2}\frac{{\rm tr}(\Sigma_{k-1})}{s}.
\end{align*}
Recall that 
\[
f_k=1-\frac{1-\zeta}{s}\frac{\beta_k\hat{\lambda}_k}{\beta_k\hat{\lambda}_{k}+\sigma^2},
\]
we have
\[
{\rm tr}(\Sigma_k) \leq  f_{k} {\rm tr}(\Sigma_{k-1}).
\]
Subsequently, 
\[{\rm tr}(\Sigma_k)\leq (\prod_{j=1}^{k}f_j) {\rm tr}(\Sigma_0).\]
Lemma \ref{proprank} shows that the rank of the covariance will not be changed by updating the covariance matrix sequentially:
${\rm rank} (\Sigma_1)=\cdots={\rm rank} (\Sigma_k)=s$. Hence, we may decompose the covariance matrix $\Sigma_k=Q Q^{\intercal}$, with $Q\in \mathbb R^{n\times s}$ being a full-rank matrix, then
${\textsf {Vol}}(\Sigma_k)={\rm det}(Q^{\intercal} Q).$
Since ${\rm tr}(Q^{\intercal} Q)={\rm tr}(Q Q^{\intercal})$, we have
\begin{align*}
&{\textsf {Vol}}^2(\Sigma_k)={\rm det}{(Q^{\intercal} Q)}
\overset{(1)}{\leq} \prod_{j=1}^{s}(Q^{\intercal} Q)_{jj}\\
&\overset{(2)}{ \leq} (\frac{{\rm tr}( Q^{\intercal} Q)}{s})^s
=(\frac{{\rm tr(\Sigma_k)}}{s})^s,
\end{align*}
where (1) follows from the Hadamard's inequality and (2) follows from the mean inequality.
Finally, we can bound the conditional entropy of the signal as 
\begin{align*}
\entropy[y_{j}, a_{j}, j \leq k]{x} 
&= \ln (2\pi e)^{s/2} {\textsf{Vol}}(\Sigma_k) \\
&\leq \frac{s}{2}\ln \{2\pi e (\prod_{j=1}^{k}f_j) {\rm tr}(\Sigma_0)\}.
\end{align*}
\end{proof}

\begin{proof}[Proof of Corollary \ref{cor:sampleSize}]
Let $\theta={\rm tr}(\Sigma)/\Vert\Sigma\Vert\geq 1$. We have that for some constant $\delta_0 > 0$, when
\[L\geq 4n^{1/2}{\rm tr}(\Sigma)(\frac{\Vert\Sigma\Vert}{\delta_0^2}+\frac{4}{\delta_0}),\]
%
with Lemma \ref{zhu2012}, we have
\begin{align*}
& P\{\Vert\widehat{\Sigma}-\Sigma\Vert \leq \delta_0\}\\
& \geq  P\{\Vert\widehat{\Sigma}-\Sigma\Vert \leq (\sqrt{\frac{2n^{1/2}(\theta+1)}{L}}+\frac{2\theta n^{1/2}}{L})\Vert\Sigma\Vert\}\\
&> 1-2n\exp(-\sqrt{n}).
\end{align*}
\end{proof}
\begin{proof}[Proof of Theorem \ref{thm:power}]
Recall that ${\rm rank}(\Sigma)=s$, $\lambda_{s+1}(\Sigma)=\cdots=\lambda_n(\Sigma)=0$.
Notice that for each step of iteration, the the eigenvalue of $\widehat{\Sigma}_k$ in the direction of $a_k$, which corresponds to the largest eigenvalue of $\widehat{\Sigma}_k$, is eliminated below threshold. Therefore, as long as the sequential algorithm continues, the largest eigenvalue of $\widehat{\Sigma}_k$ is exactly the $(k+1)$th largest eigenvalue of $\widehat{\Sigma}$.
Now that $\delta_0\leq \frac{1}{4^{s+1}}\frac{\varepsilon^2}{\chi_n^2(p)}$ with Lemma \ref{lemmaeig} and Lemma \ref{propdelta},
$$\vert \hat{\lambda}_{k}-\lambda_k(\Sigma) \vert\leq \delta_0,\ {\rm for}\ k=1,\ldots,s,$$
$$\vert \hat{\lambda}_{j} \vert\leq \delta_0\leq \frac{\varepsilon^2}{\chi_n^2(p)}-\delta_{s},\ {\rm for}\ k=s+1,\ldots,n.$$
Notice that in the ideal case with no perturbation, the aim of each measurement is to decrease the eigenvalue of a particular direction to $\varepsilon^2/{\chi_n^2(p)}$.
Suppose in the ideal scenario, the algorithm stops after $K\leq s$ steps of iteration.
Hence, $$\lambda_1(\Sigma)\geq\cdots\geq \lambda_K(\Sigma)>\frac{\varepsilon^2}{\chi_n^2(p)},$$
$$\lambda_{s}(\Sigma)\leq\cdots\leq\lambda_{K+1}(\Sigma)\leq\frac{\varepsilon^2}{\chi_n^2(p)}.$$
Therefore, the power needed in the ideal case is
$$P_{\rm ideal}=\sum_{k=1}^{K}(\frac{\chi^2_n(p)}{\varepsilon^2}-\frac{1}{\lambda_k(\Sigma)})\sigma^2.$$
In the noisy case, for the first $K$ steps of measurements, $1\leq k\leq K$, we choose the power $$\beta_k=\sigma^2(\frac{1}{\frac{\varepsilon^2}{\chi_n^2(p)}-\delta_s}-\frac{1}{\hat{\lambda}_{k}}).$$
We have
$$\frac{\sigma^2}{\beta_{k-1}\hat{\lambda}_{k-1}+\sigma^2}\hat{\lambda}_{k-1}= \frac{\varepsilon^2}{\chi_n^2(p)}-\delta_s.$$
For the steps $K+1\leq k\leq s$,
\begin{align*}
\beta_k
&=\max\{0,\sigma^2(\frac{1}{\frac{\varepsilon^2}{\chi_n^2(p)}-\delta_s}-\frac{1}{\hat{\lambda}_k})\}\\
& \leq \sigma^2(\frac{1}{\frac{\varepsilon^2}{\chi_n^2(p)}-\delta_s}-\frac{1}{\frac{\varepsilon^2}{\chi_n^2(p)}+\delta_0})\\
&\leq \sigma^2 \frac{(4^s+1)\delta_0}{(\frac{\varepsilon^2}{\chi_n^2(p)}+\delta_0)(\frac{\varepsilon^2}{\chi_n^2(p)}-4^s \delta_0)} \leq \frac{20}{51}\frac{\chi_n^2(p)}{\varepsilon^2}\sigma^2.
\end{align*}
With Lemma \ref{lemmaeig}, all eigenvalues of $\Sigma_K$ are no greater than $$\frac{\varepsilon^2}{\chi_n^2(p)}-\delta_s+\lambda_1{(E_s)}=\frac{\varepsilon^2}{\chi_n^2(p)}.$$
And the total power
\begin{align*}
P_{\rm mismatch}
&= \sum_{k=1}^{s}\beta_k\\
&\leq\sigma^2\{\sum_{k=1}^{K}(\frac{1}{\frac{\varepsilon^2}{\chi_n^2(p)}-\delta_s}-\frac{1}{\hat{\lambda}_{k}})+\frac{20(s-K)}{51}\frac{\chi_n^2(p)}{\varepsilon^2}\}.
\end{align*}
In order to achieve precision $\varepsilon$ and confidence level $p$,
the extra power needed is upper bounded as
\begin{align*}
&P_{\rm mismatch}-P_{\rm ideal}\\
&\leq \sigma^2\{\sum_{k=1}^{K} (\frac{1}{3}\frac{\chi_n^2(p)}{\varepsilon^2}+\delta_0\frac{1}{\lambda_k^2}) +\frac{20(s-K)}{51}\frac{\chi_n^2(p)}{\varepsilon^2}\}\\
&\leq\sigma^2\{\frac{1}{4^{s+1}}\frac{\varepsilon^2}{\chi_n^2(p)}\sum_{k=1}^{K}\frac{1}{\lambda_k^2}+\frac{20s-3K}{51}\frac{\chi_n^2(p)}{\varepsilon^2}\}\\
&<(\frac{20}{51}s-(\frac{1}{17}-\frac{1}{4^{s+1}})K)\frac{\chi_n^2(p)}{\varepsilon^2}\sigma^2\\
&\leq (\frac{20}{51}s+\frac{1}{272}K)\frac{\chi_n^2(p)}{\varepsilon^2}\sigma^2.
\end{align*}
\end{proof}

\begin{proof}[Proof of Theorem \ref{thm:cov-sketch}]
Let $\theta={\rm tr}(\Sigma)/\Vert\Sigma\Vert\geq 1.$ With Lemma \ref{est}, let $\tau={M\delta_0}/{C_2}$,
the choice of $M$,$N$, and $L$ ensures that $\Vert\eta\Vert_1\leq{M\delta_0}/{C_2}$ with probability at least $1-{2}/{n}-{2}/{\sqrt{n}}-2n\exp(-\sqrt{n}))$.
By applying Lemma \ref{sketching} and noting that the rank of $\Sigma$ is exactly $s$, we have
$$\Vert \widehat{\Sigma}-\Sigma \Vert_F\leq \delta_0.$$
Therefore, with probability exceeding $1-2/n-{2}/{\sqrt{n}}-2n\exp(-\sqrt{n}))-\exp(-c_0c_1ns),$
$$\Vert \widehat{\Sigma}-\Sigma \Vert\leq \Vert \widehat{\Sigma}-\Sigma \Vert_F\leq \delta_0.$$
\end{proof}

\end{document}